%% file: main.tex
\documentclass{bmvc2k}
\usepackage{amsfonts,amsmath,amssymb,amsthm}
\usepackage{bm}
\usepackage{booktabs}
\usepackage{capt-of}
\usepackage{lineno}
\usepackage{lipsum}
\usepackage{multicol}
\usepackage{multirow}
\usepackage{placeins}
\usepackage{siunitx}
\usepackage{subfigure}
\usepackage{soul}
\usepackage{tabularx}
\usepackage{floatrow}
\usepackage{cleveref}

\newfloatcommand{capbtabbox}{table}[][\FBwidth]
\newtheorem{lemma}{Lemma}

\makeatletter
\renewcommand{\paragraph}{%
  \@startsection{paragraph}{4}%
  {\z@}{0.5em}{-1em}%
  {\normalfont\normalsize\bfseries}%
}
\makeatother

\title{AutoCorrect: Deep Inductive Alignment of Noisy Geometric Annotations}

\addauthor{Honglie Chen}{hchen@robots.ox.ac.uk}{1}
\addauthor{Weidi Xie}{weidi@robots.ox.ac.uk}{1}
\addauthor{Andrea Vedaldi}{Vedaldi@robots.ox.ac.uk}{1}
\addauthor{Andrew Zisserman}{az@robots.ox.ac.uk}{1}
\addinstitution{
 Visual Geometry Group\\
 University of Oxford\\
 Oxford, UK
}
\runninghead{CHEN, XIE, VEDALDI, ZISSERMAN}{AUTOCORRECT: DEEP INDUCTIVE ALIGNMENT}

\def\ie{\emph{i.e}\bmvaOneDot}

\begin{document}
\maketitle
\begin{abstract}
We propose AutoCorrect, a method to automatically learn object-annotation alignments
from a dataset with annotations affected by geometric noise.  The
method is based on a consistency loss that enables deep neural
networks to be trained,  given only noisy annotations as input, to
correct the annotations.  When some noise-free annotations are
available, we show that the consistency loss reduces to a stricter
self-supervised loss.  We also show that the method can implicitly
leverage object symmetries to reduce the ambiguity arising in
correcting noisy annotations.  When multiple object-annotation pairs
are present in an image, we introduce a spatial memory map that allows
the network to correct annotations sequentially, one at a time, while
accounting for all other annotations in the image and corrections
performed so far.  Through ablation, we show the benefit of these
contributions, demonstrating excellent results on geo-spatial imagery.
Specifically, we show results using a new Railway tracks dataset as
well as the public INRIA Buildings benchmarks, achieving new
state-of-the-art results for the latter.
\end{abstract}

\input{sec/Introduction}
\input{sec/Related_Works}
\vspace{-9pt}
\input{sec/Proposed_method}
\vspace{-3pt}
\input{sec/Experiment}

\vspace{-12pt}
\input{sec/Conclusion}

\paragraph{Acknowledgement.}
We thank Kai Han, Erika Lu and Tengda Han for proofreading. Financial support was provided by 
the EPSRC Programme Grant Seebibyte EP/M013774/1.

\bibliography{shortstrings,vgg_local,vgg_other}
\end{document}

%% file: sec/Introduction.tex
\section{Introduction}\label{introduction}

Digital images are nowadays collected in enormous quantities.  An
important example is geo-spatial data, collected continuously by
satellites, and containing a wealth of information useful for urban
planning, crop and forest management, disaster relief, climate
modelling, and many other applications.  However, the scale of such
datasets requires automated processing via machine learning and, while
machine learning methods are increasingly powerful, providing annotations
manually to train them can be prohibitively expensive.

The annotation costs may be substantially reduced if labels need not be very accurate.
In this case, it is sometimes possible to \emph{recycle} annotations that were not collected specifically for the images at hand.
With geo-spatial data, for instance, there are publicly available maps (e.g.~OpenStreetMap~\cite{OpenStreetMap}, Google Maps~\cite{GoogleMap}) that can provide annotations for large areas of the planet for free.
However, while maps are generally accurate, they usually fail to match satellite images exactly due to various issues.
To list a few:
1)~maps do not capture the 3D structure of features such as buildings or vegetation, leading to misaligned annotations due to viewpoint variations;
2)~maps may not be temporally synchronized with the satellite data, thus failing to account for variations in buildings, roads and vegetation;
3)~features recorded in a map (e.g.~subways) may not necessarily be visible in images and vice-versa.
\Cref{Fig.data} shows examples of noisy geometric labels obtained from these data sources 
in the \emph{INRIA buildings} 
and 
our new \emph{Railway tracks} datasets,  and compares  them with the manually-corrected versions.

\input{fig_tab/Fig-data}

Noisy labels can severely impact the quality of learned object detectors, as shown in  satellite/aerial segmentation~\cite{Mnih12,Saito16,Alshehhi17}~and detection~\cite{Laptev2000,Hu2007}.
Hence, in this paper, we consider the problem of improving noisy labels to reduce or eliminate the impact of such noise on learned models.
Our method, \emph{AutoCorrect}, is mostly concerned with registration noise, which is usually the predominant noise type in geo-spatial data~(\Cref{Fig.data}).
We build a model that takes a set of images and misaligned object annotations as input and shifts the annotations to their correct image locations.

There are several challenges.
Satellite images usually contain multiple occurrences of the same object types, which may lead to association errors.
Geo-spatial images capture the top of tall objects such as building and trees, whereas maps annotate their base.
Finally, tall objects (e.g.\ trees in Figure~\ref{Fig.tac} or buildings) can occlude other objects or cast significant shadows,
so that some objects annotated in the map may effectively be invisible.

Given an image and a set of object annotations,  \emph{AutoCorrect} sequentially registers each annotation to its corresponding object occurrence by estimating an instance-level transformation.
This is much more flexible than existing works that seek a single image-level transformation and allows us to obtain substantial improvements compared to these (indeed, as will be seen in the results, the annotations are displaced independently per object, and a single image-level correction will not suffice).
However, this comes with several challenges.
First, the model may not have access to any noise-free annotation, or at least not be aware of which ones are noise-free, making the correction process ambiguous.
Second, there usually are several objects in each image, which means that the model must generalise to an arbitrary number of object occurrences whilst avoiding errors due to duplicate associations.

We solve the first problem by combining a geometric consistency loss, which is valid even if the ground-truth annotations are unknown, with a self-supervised loss, which is reliable for annotations with a small amount of noise.
We also show that the symmetry of certain objects such as roads provides an implicit constraint that makes registering annotations much less ambiguous.
We solve the second problem by introducing a \emph{spatial memory map} which represents all image annotations and reflects all previously-applied corrections.

%% file: fig_tab/Fig-data.tex
\begin{figure*}[t]
\centering
\subfigure[]{\includegraphics[width=0.23\textwidth]{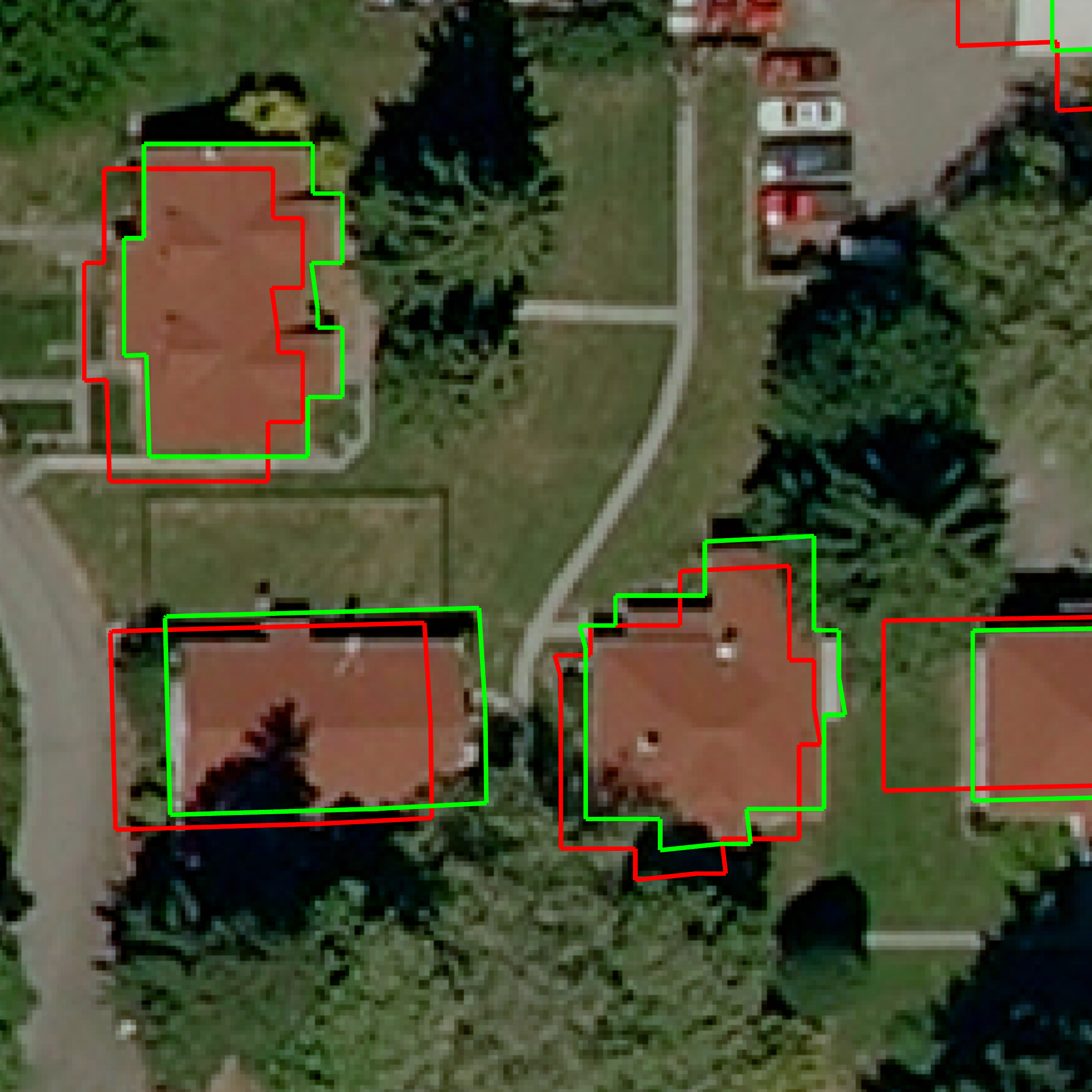}\label{Fig.im}}
\subfigure[]{\includegraphics[width=0.23\textwidth]{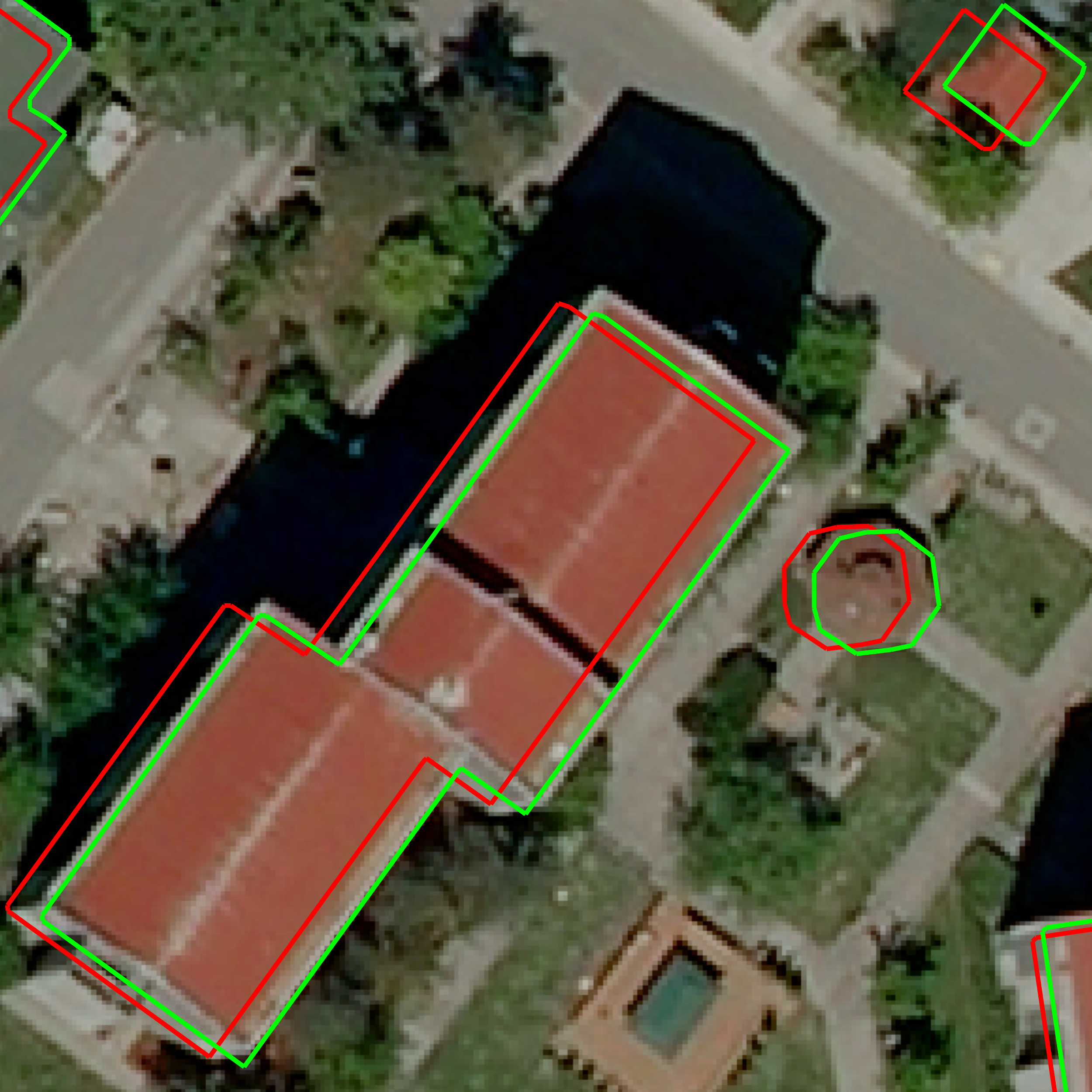}\label{Fig.ta}}
\subfigure[]{\includegraphics[width=0.23\textwidth]{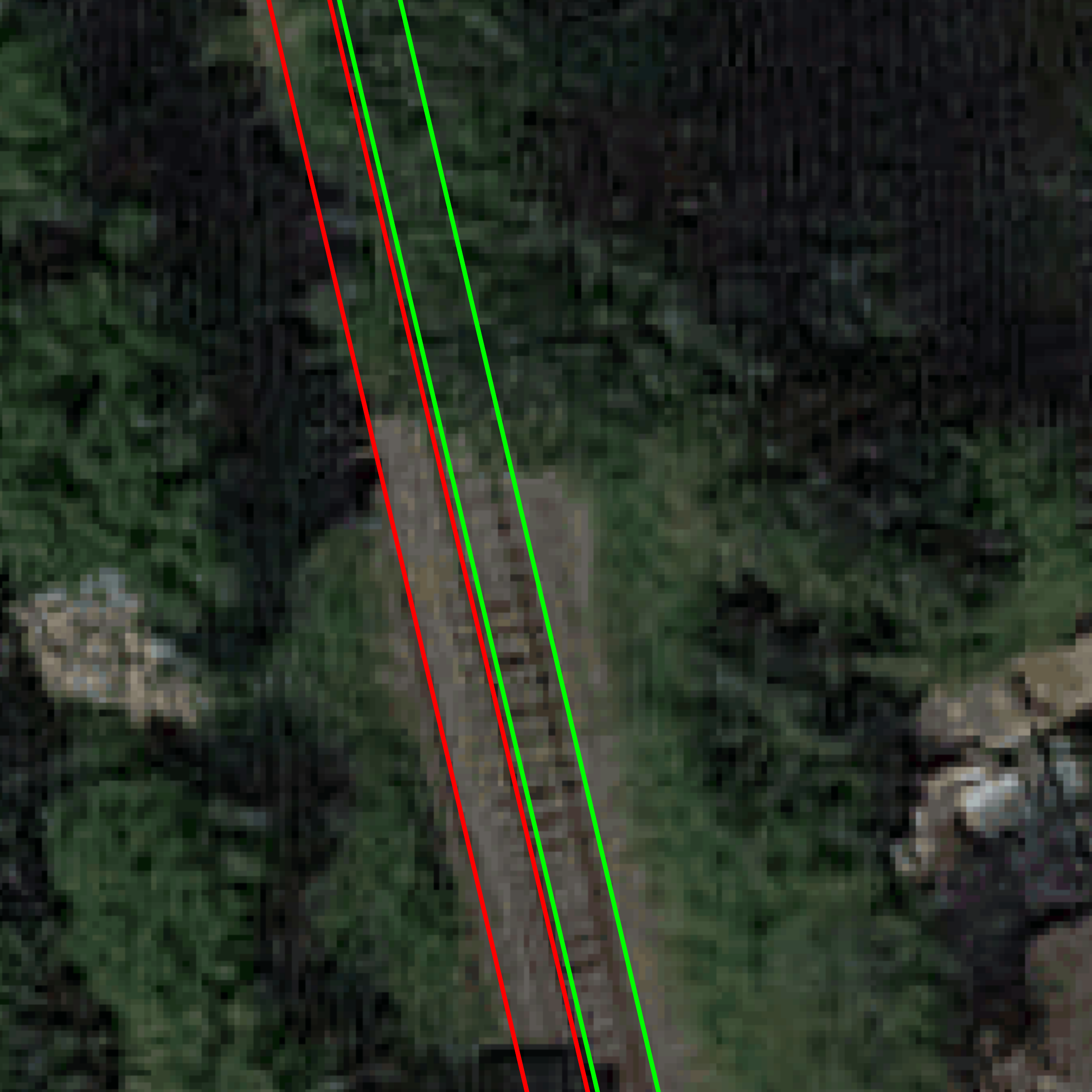}\label{Fig.tac}}
\subfigure[]{\includegraphics[width=0.23\textwidth]{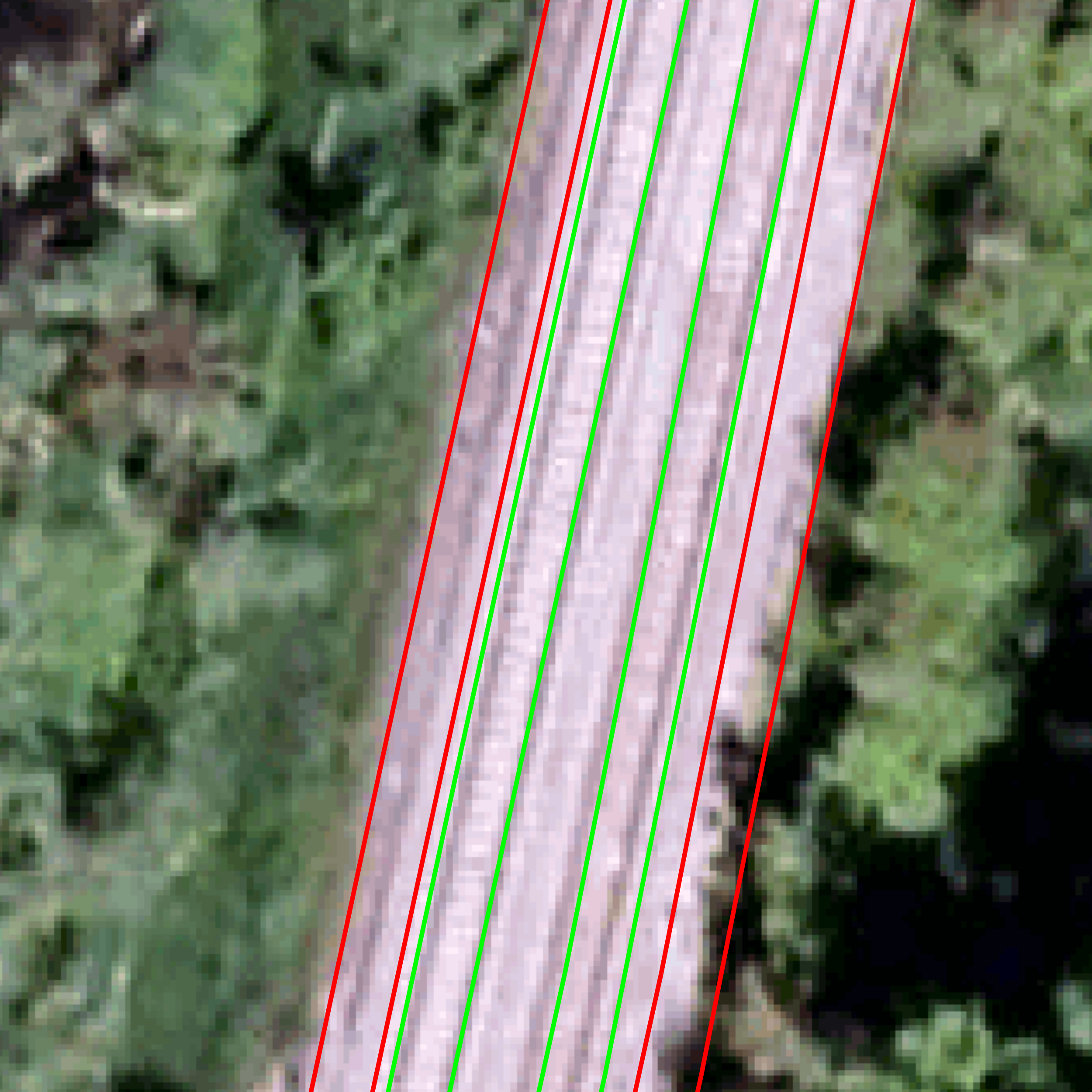}\label{Fig.im2}}
\vspace{-1.5em}
\caption{Example aerial images with noisy labels (Red) and accurate labels (Green).
(a) and (b) are extracted from the \emph{INRIA buildings} dataset. (c) and (d) are examples in the \emph{Railway tracks} dataset.
The original labels~(Red) demonstrate the clear registration noise. 
The cleaned labels (Green) show the corrections we aim to achieve (Human corrected).}\label{Fig.data}        
\end{figure*}

%% file: sec/Related_Works.tex
\section{Related work}
\label{Related work}
\paragraph{Image alignment.}
Two very related works~\cite{Girard18,Zampieri18} have shown good alignment performance by training a CNN
to predict a displacement field between a map and an image. \cite{Zampieri18} uses a 
multi-scale CNN, and~\cite{Girard18} improves performance by training jointly for both alignment and segmentation.
We compare to their results (and improve over them) in Section~\ref{experiments}. 

\paragraph{Inductive models and spatial memory.}
Explicit decomposition into repeated sub-tasks and recursively
solving the problem have been applied in neural
programming~\cite{Zaremba16,Cai17,Reed16a} and many visual
tasks~\cite{Li18,Romera-Paredes15,Kowalski17,Carreira15,Oberweger15,Gupta18}.
In~\cite{Kowalski17}, each stage predicts a landmark transformation that
updates the keypoints iteratively.  Similarly, an updater function is
formulated in~\cite{Oberweger15} for hand pose alignment.  
\cite{Gupta18} proposes an
inductive RNN to localise visual objects which can generalise to an
arbitrary number of inputs. Many of these methods use a form of spatial memory, though this
isn't always made explicit. Others have used spatial memory for 
interactive image
segmentation~\cite{Li18},  and context reasoning in object detection~\cite{Chen17a}.

\paragraph{Cycle consistency.}
Assessing performance via cycling between two or more samples is a
commonly used technique in computer vision. Many successful tasks like
optical flow (with forward-backward consistency)~\cite{Sundaram10}, co-segmentation~\cite{Wang14}, image
matching~\cite{Zhou15a,Zhou16}, image translation~\cite{Zhu17}
and domain adaptation~\cite{Hoffman18} have shown its
effectiveness. 
We introduce here a geometric-consistency loss: that within an image,
misaligned annotations should be
able to transform back to a single unique position.

\paragraph{Learning with imperfect annotation.}
Most works on learning with imperfect annotations have considered
classification, rather than registration. Examples include having a
small set of clean samples (as well as many
noisy)~\cite{Xiao15,Veit17}, using robust loss
functions~\cite{Ghosh17,Patrini17},  or using a top-k
loss~\cite{Berrada18}.

%% file: sec/Proposed_method.tex
\section{Approach}
\label{Map Feature Extraction}

Our goal is to train deep networks for the detection of visual objects while relying on noisy annotations.
While the approach is fairly general, we apply it to the detection of objects such as building and roads in geo-spatial images, where noisy annotations can be extracted from on-line data repositories such as mapping services.
The mismatch between annotations and images is sometimes large, as shown in \Cref{Fig.data}.
Na\"\i{}vely training a model with these annotations leads to inaccurate predictions.

\input{fig_tab/Fig-consistency}

There are two main challenges. First, all annotations are potentially noisy and thus it is not clear how the noise can be identified and removed.
Second, as different objects in the image may be misaligned in different ways, we must enable instance-level corrections while handling an arbitrary number of object instances per image.
We address these challenges in three ways.
First, we use a self-supervised consistency loss based on the fact that multiple perturbations of the same label must always map to the same noiseless label.
Second, we show that the intrinsic symmetry of certain visual objects provides a powerful implicit constraint that can reduce the ambiguity in the annotation clean-up process.
Third, we introduce the idea of inductive alignment, adjusting annotations one instance at a time, sequentially, keeping track of the algorithm state by means of a \emph{spatial memory map}.
This is implemented by a recurrent neural network (RNN), which applies the same alignment logic to each annotation, but accounting for annotations already processed.

\subsection{Single instance alignment}\label{s:single}

We start by describing a neural network architecture that can predict a translation and rotation for an individual object annotation in order to better align it to the image content. 
Note that, while this task may sound similar to object detection, it is in fact much easier as the annotation cues us to the \emph{existence} and rough location of an object.

At each step, the input to the model is a concatenation of the RGB image $I\in\mathbb{R}^{3\times H\times W}$ with a scalar label map $y\in\{0,1\}^{H\times W}$ which encodes the annotation as a binary image.
We know that the annotations can potentially be noisy, so we wish to learn a predictor function that outputs the transformation (\ie 2 scalars for translation and 1 for rotation) to align the image and annotation.
This is implemented using a CNN that takes as input $I$ and $y$ and outputs a transformation $t$:
\begin{align}\label{e:model}
t = \Phi(I,y).
\end{align}
The corrected annotation $\hat{y} = t \cdot y$ is expressed as the transformed version of the annotation $y$ by the predicted transformation $t \in G$, where $G$ is a group of transformations $\mathbb{R}^2\rightarrow\mathbb{R}^2$ such as 2D similarities.
The symbol $\cdot$ denotes warping an image by a transformation.
If the annotation is noise-free, $t$ is expected to be an identity matrix and $\hat{y} = y$. If the annotation is noisy, 
the corrected annotation $\hat{\mathbf{y}}$ should approximate the underlying noise-free annotation $\mathbf{y}_{\text{gt}}$, which however is never observed during training.

Model~\eqref{e:model} has several useful geometric properties:
\begin{lemma}\label{l:one}
If $y_\text{gt}$ is the ground-truth annotation for image $I$ and a perfect $\Phi$ is available, then
$
\Phi(I,y_\text{gt})= \mathbf{1}
$
is the identity transformation. Furthermore, for all invertible transformations $g \in G$, we have
$
\Phi(g\cdot I,y) = g \Phi(I,y)
$
and
$
\Phi(I,g\cdot y) = \Phi(I,y) g^{-1}.
$
\end{lemma}
The lemma is easy to prove once we note that, if $y_\text{gt}$ is the ground-truth label of image $I$, then $g \cdot y_\text{gt}$ is the ground-truth label of image $g \cdot I$.
From this lemma, we can also see that any annotation that can be recovered from an image must have the same \emph{symmetries} as the image itself.

\begin{lemma}
Let $\hat{y} = \Phi(I,y) \cdot y$ be the annotation reconstructed from image $I$ using model~\eqref{e:model} and assume that $m\in G$ is a symmetry of the image, \ie $I = m \cdot I$.
Then the reconstructed annotation has the same symmetry, in the sense that $\hat{y} = m \cdot\hat{y}$.
\end{lemma}

\begin{proof}%
\hfill %
$
\hat{y}
= \Phi(I,y) \cdot y
= \Phi(m\cdot I,y) \cdot y 
= m \Phi(I,y) \cdot y
= m  \cdot \hat{y}.
$\hfill%
\end{proof}

This lemma shows that annotations can be predicted from images only if they have the same symmetries as the images.
For example, if the model labels a straight road with a line, then the line must coincide to the road axis of symmetry.
Hence image symmetries implicitly constrain the predictor~\eqref{e:model} (in the example of the road, the correction must move the line onto the visual axis of symmetry of the road), reducing the ambiguity in registering the annotation.
Note that this effect does not require specific images to be exactly symmetric; rather, it suffices that the object category is \emph{statistically} symmetric (for example it is not possible to tell the direction of a road even if there are a few trees on one side, making the image asymmetric).

If we assume all annotations are correct, \ie $y = y_{gt}$, 
then \Cref{l:one} can be used to train model~\eqref{e:model} via self-supervised learning.
The idea is to perturb the noise-free annotations synthetically by applying a random transformation $g \in G$ to the annotation $y = y_\text{gt}$.
From~\Cref{l:one}, and using the assumption $y = y_\text{gt}$, we have
$
\Phi(I, g \cdot y) = \Phi(I,y) g^{-1} = 1 g^{-1} = g^{-1}.
$
We may capture this constrain in the {\em self-supervised loss}:
\begin{equation}\label{e:loss1}
J_s = \|g^{-1} - \Phi(I, g\cdot y)\|^2
\end{equation}
However, in our case $y_{gt}$ is unknown so this loss can be used only as an approximation.
In this case, the constraint can be written in term of \emph{relative} transformations.
To this end, consider applying two random transformations $y_1 = g_1 \cdot y$ and $y_2 = g_2 \cdot y$ to the annotation $y$.
From~\Cref{l:one}, we have
$
\Phi(I,g_1\cdot y) g_1 = \Phi(I,y)  = \Phi(I,g_2\cdot y) g_2.
$
This can be written as a {\em consistency loss}:
\begin{equation}\label{e:loss2}
J_c = \|t_1 g_1 - t_2 g_2\|^2,
\quad
t_1 = \Phi(I,g_1\cdot y),
\quad
t_2 = \Phi(I,g_2\cdot y).
\end{equation}
Intuitively, when two random transformations operate on one annotation, an ideal alignment model should be able to transform the annotation back to the same position, as $y_{gt}$ is unique.

Overall, to train models on noisy data, we therefore consider a weighted combination $J = \alpha_s \cdot J_s + \alpha_c \cdot J_c$ (details in \Cref{sec:training}).

\subsection{Inductive alignment}\label{regress}

\input{fig_tab/Fig-net}

A na\"\i{}ve implementation of model~\eqref{e:model} may align single object instances well, but it would fail when an image contains multiple object occurrences, especially when, as in satellite images, these are spatially close and similar in appearance.
In particular, independent alignment may cause different noisy annotations to be incorrectly associated to the same object occurrence.
To tackle this challenge, we introduce an inductive alignment model which uses an external \emph{spatial memory map} to make the algorithm aware of all annotations present in the image as well as to keep track of all correction processed so far.
Formally, given a training image $I$ with $n$ object annotations $\mathbf{y}=(y_1,\dots,y_n)$, our goal can be seen as estimating the joint posterior density of the transformation matrix for all the noisy object annotations~$p({\{t_i\}}_{i=1}^{n}|I, \mathbf{y})$.
Rather than modelling multiple object annotations simultaneously, we break this down as \emph{sequence} of simpler steps, in which a single transformation is predicted at a time, conditioned on the previous decisions, resembling an autoregressive model.
Formally, this autoregressive model can be written as:
$
P(t_1, t_2, \ldots, t_n| I, \mathbf{y}) =
P(t_1 | I, \mathbf{y})
P(t_2 | I, t_1, \mathbf{y})
\cdots
P(t_n | I, \{t_i\}_{i=1}^{n-1},\mathbf{y}).
$
Note that this process requires learning a sequence of models $P(t_i | I, t_1,\dots,t_{i-1}, \mathbf{y})$.
Directly parameterising the relations among transformations is difficult and results in a model which is rather opaque;
instead, we propose to summarise the effect of conditioning on the previous corrections $t_1,\dots,t_{i-1}$ via a \emph{spatial memory map} $M_{i-1}$, 
ideally, the memory map should represent all annotations and 
corrections performed so far \emph{except} the annotation $y_i$ that is currently being processed,
formally:
\begin{equation}
P(t_i |I, t_1,\dots,t_{i-1},\mathbf{y})=P(t_i |I,M_{i-1},y_i), \qquad
M_i = M_{i-1} + t_{i} \cdot y_{i} - y_{i+1}.
\end{equation}

An explicit example is illustrated in~\Cref{Fig.net}, showing four railway track annotations to be corrected.
The algorithm starts with four binary masks, each coding one of the noisy railway annotations,
at the very \emph{first} step, 
the memory $M_0$ is composed of three annotations~($y_2, y_3, y_4$),
and $y_1$ is concatenated as additional input to the network $\Phi$.
Then, the first annotation $y_1$ is corrected by predicting the rigid transformation $t_1$, 
and the memory $M_1$ is updated by adding the image of $t_1 \cdot y_1$ and removing the image of annotation $y_2$, 
readying for the next cycle.
The induction process ends at $M_4$, where all tracks have been effectively corrected by the model.
Note that the order we align instances is from left to right and bottom to top. 

\subsection{Implementation details}
\label{sec:training}

Consider a training image $I$ with $n$ noisy object annotations $\mathbf{y}=(y_1,\dots,y_n)$.
Annotations are perturbed by applying random transformations $g \cdot y_i$ where $g$ is the composition of a translation of up to $25$px in each direction and a rotation of up to 5 degrees (clockwise or anticlockwise) as this was found to be commensurate to the maximum amount of noise in the geo-spatial datasets we used for assessment.

During early training, we set the gating parameters in the joint objective function $J$ as $\alpha_s = \alpha_c = 1$.
This ensures the model converges quickly to an approximate solution within a few pixels of the ground-truth annotation, despite the fact that annotations are noisy so that term $J_s$ in the objective function is not exactly valid.
In a second phase, when the model is close to the final solution, the terms $J_s$ and $J_c$ start to be in conflict for the annotations that contain the largest amount of noise.
Hence the coefficient $\alpha_s$ and $\alpha_c$ are adjusted as follows:
\begin{align}\label{equation_gating}
\begin{cases}
\alpha_s =0,~\alpha_c =1 & \text{if
$
\min(
\operatorname{IoU}(t_1 g_1 \cdot y,y), 
\operatorname{IoU}(t_2 g_2 \cdot y,y)) < 0.2
$
},  \\
\alpha_s =1,~\alpha_c =0 & \text{otherwise},\\
\end{cases}
\end{align}
where $\operatorname{IoU}$ denotes the standard \emph{Intersection over Union} measure.
This states that when any of the predicted corrections is far away from the given label, the label is expected to contain a large amount of noise, only the consistency loss $J_c$ is applied; otherwise only the stricter self-supervised loss $J_s$ is used.

\paragraph{Architecture and optimisation.}

The proposed \emph{AutoCorrect} model uses as backbone architecture the VGG-M network~\cite{Chatfield14} with minor modifications (details in the supplementary materiel).
The network is trained using the Adam optimiser at an initial learning rate of $10^{-4}$, which is divided by $10$ after the training error plateaus.

%% file: fig_tab/Fig-consistency.tex
\begin{figure*}[t]
\centering\includegraphics[width=0.9\textwidth]{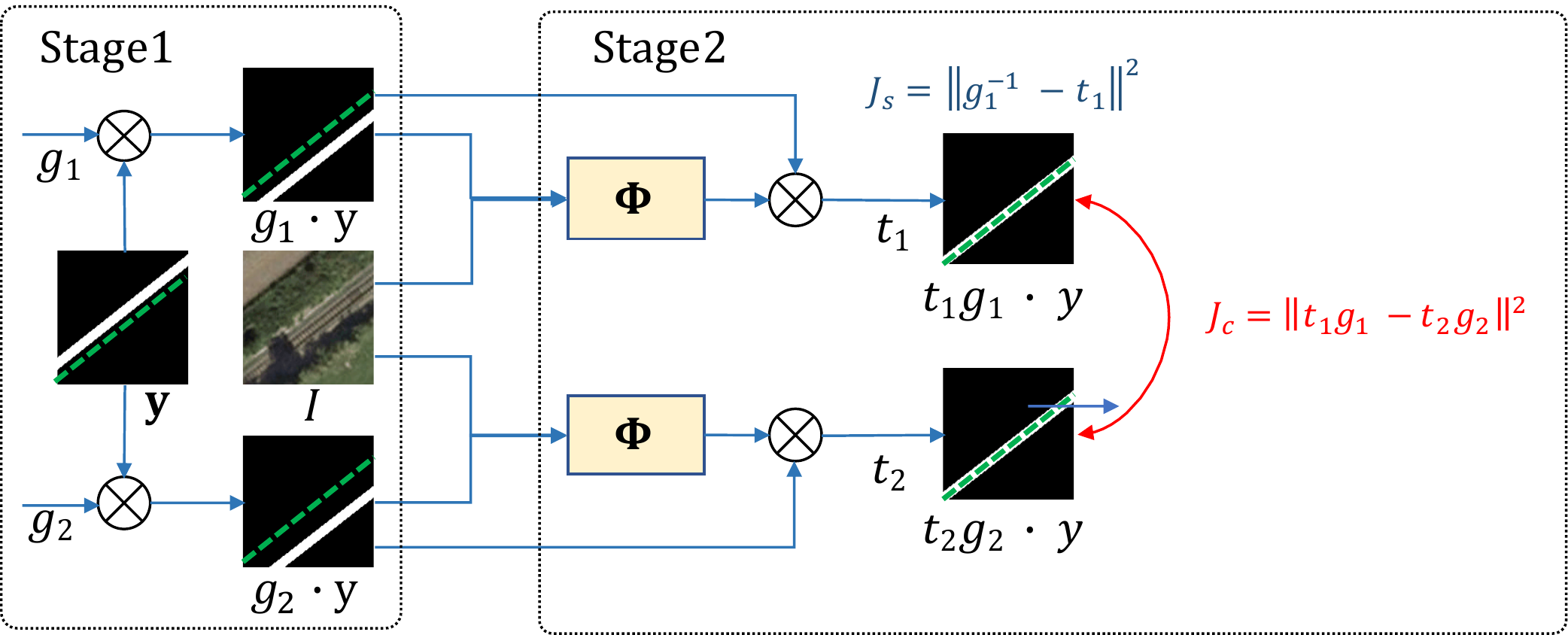}
\vspace{-1.2em}
\caption{AutoCorrect architecture.
The green dotted line shows the ground-truth label for the example image of a railway track.
In Stage~1, given an image-label pair $(I,y)$, the noisy annotation $y$ is further perturbed by applying random transformations $g_1$ and $g_2$.
In Stage~2, the network $\Phi$ computes corrections $t_i = \Phi(I, g_i\cdot y)$, $i=1,2$, producing corrected labels $(t_i g_i) \cdot y$ which must satisfy the consistency equation $J_c=0$ (see text).
If $y$ is known to be a noise-free, then we can set $g_2 = t_2 = 1$ reducing $J_c$ to the stricter constraint $J_s=0$.}\label{Fig.consistency}
\end{figure*}

%% file: fig_tab/Fig-net.tex
\begin{figure*}[t]
\centering\includegraphics[width=1.0\textwidth]{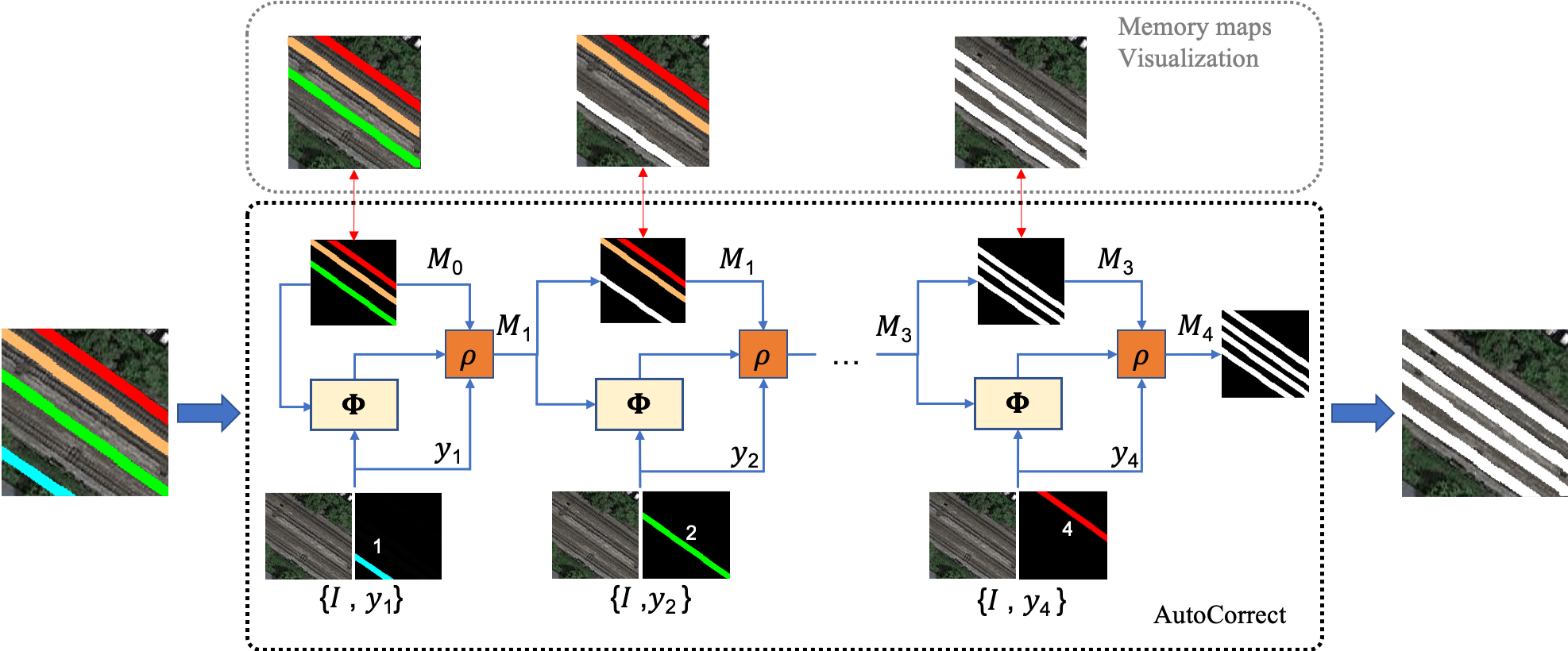}
\vspace{-2em}
\caption{Correcting annotations sequentially using a memory map.
Note, we demonstrate our correction process in the \emph{AutoCorrect} box, whereas the top \emph{Memory maps Visualisation} box highlights the corrected annotation (white) on the satellite image.
In detail, at each step, the input to the network is the concatenation of the RGB image $I$, the image of the annotation $y_i$ to be corrected, and a memory $M_{i-1}$ representing all other annotations, part of which have already been corrected (we colour-code annotations not yet corrected).
An update function $\rho$, is applied to obtain the correction $y_i \mapsto t_i \cdot y_i$ and the latter is used to update $M_{i}$.
}\label{Fig.net}
\end{figure*}

%% file: sec/Experiment.tex
\section{Experiments}\label{experiments}

The experiments thoroughly assess our \emph{AutoCorrect} method on two benchmark datasets:
our own Railway tracks dataset and the INRIA buildings dataset.
The new Railway tracks dataset will be released at \url{http://robots.ox.ac.uk/~vgg/research/autocorrect/}.
\subsection{Datasets and evaluation}
\paragraph{Railway tracks dataset.}
The \emph{Railway tracks} dataset was obtained by extracting views of railways in the UK region from Google Maps.
We used zoom level $19$, which corresponds to approx.~$0.5$ meter/pixel (this is the minimum zoom level at which railway tracks can be resolved) and results in images with a $640 \times 640$ pixel resolution.
The dataset contains approximately $35$k overhead images of the tracks.
Binary mask annotations are provided by Google Maps to indicate the position of the railway tracks;
however, the annotations are not perfectly aligned with the images~(shown in~\Cref{Fig.data}).
In order to evaluate the effectiveness of the
self-supervised learning loss, the consistency loss, and the spatial
memory map, we  manually identify $4,\!000$ images for which
railway annotations are accurate. We use these in the experiments by synthetically adding noise to these ground-truth
annotations.

\paragraph{INRIA buildings dataset.}
The \emph{INRIA buildings} dataset contains 360 images of $5,\!000 \times 5,\!000$ pixels.
This dataset may seem small compared to other deep learning datasets, but as each image has a large spatial footprint, it contains a large number of buildings ($13,\!614$ buildings just in the test split).
In order to directly compare with prior work, we adopt the same data and evaluation protocol of~\cite{Girard18}.

\paragraph{Evaluation metrics.}

In order to evaluate the effectiveness of the proposed \emph{AutoCorrect} model,
For \emph{Railway tracks} dataset, we assess railway alignment using the standard IoU measure between the image of a noise-free label and the predicted correction of a noisy label.
For the \emph{INRIA buildings} dataset, in order to compare with existing work, we adopt the standard protocol and report results using the \emph{Percentage of Correct Keypoints}~(PCKs) metric.
The reason we apply the IoU measure on railway tracks is that railway tracks tend to be straight and long, so that, unlike for buildings, it is difficult to define keypoints.
Note that IoU is very sensitive for thin structures such as railroads.
\input{fig_tab/table1}

\subsection{Railway tracks results}
\paragraph{Synthetic annotation noise.} 
In the following, we use the $4,\!000$ images with ground-truth (i.e.\ correct) annotations, 
split as $3,\!000$ for training the \emph{AutoCorrect} network and $1,\!000$ for testing.
With these image-annotation pairs, 
we aim to perform controlled experiments on evaluating the effectiveness of the proposed components.
First, we assess the effectiveness of the
spatial memory map by training our model using only the $3,\!000$ noise-free annotations and the self-supervised loss.  
Then, to evaluate the robustness of the consistency loss against different levels of noise,  
we intentionally replace the noise-free annotations with perturbed ones in training set, 
and train three sets of models, with  resp.~$0\%,20\%,40\%$ noisy annotations,  and using or not using the consistency loss.  
During the testing stage, we artificially perturb the $1,\!000$ testing images three  times, 
and apply our models to correct the perturbed  testing annotations.
All artificial perturbations are composed of a random translation up to 25px in each direction and a random rotation up to 5 degrees~(clockwise or anti-clockwise).

As shown in \Cref{Fig.results}, models A-G are trained on only 3k images with noise-free labels or with the injection of synthetic noise in part of those. H, I, J are trained on real annotation noise.
{First}, to show that our \emph{spatial memory map} plays an important role in the instance alignment, we compare models A and B:
the performance gap is significant~($0.321$ vs. $0.425$ IoU), as the spatial memory map gives important contextual information.
{Second}, comparing models B and C shows that the consistency loss is beneficial even when training on the noise-free subset of the data.
We conjecture that this is because the consistency loss acts as a regularizer.
{Third}, to verify the effectiveness of the consistency loss in dealing with noisy data, we note that as the noise ratio is increased~(models D and F), the performance of the model that uses only the self-supervised loss starts to drop dramatically~($0.404$ and $0.369$ IoU);
however, the transformation consistency loss improves the robustness to noise significantly~(models E and G, $0.429$ and $0.381$ IoU).
Note that, when the noise ratio is around $20\%$, model E actually performs about as well as model C, which learned on noise-free annotations.
This shows that models trained with transformation consistency can discount almost entirely moderate amounts of noise.

\paragraph{Natural annotation noise.}
After demonstrating the concept in these controlled experiments, we now 
train the network using the entire dataset (which we estimate to contain $40\%$ of labels with significant geometric distortion), using either $20$k or $35$k images and switching the consistency loss on and off to test its effectiveness once more.
Similar to synthetic annotation noise, we artificially perturb the 1,000 testing images to evaluate models trained on natural annotation noise.
The models I and J ($20$k/$35$k images, $0.435$/$0.445$ IoU) show that, even with substantial real annotation noise ($\sim$40\%), our model reaches similar or superior performance to using a manually filtered dataset (C, 3k images, 0.436 IoU) with no annotation noise.
The advantage is that, while datasets I and J are large, they are obtained ``for free'' without any manual filtering.

\subsection{INRIA buildings dataset results}
To evaluate our alignment method on the \emph{INRIA buildings}  dataset, 
we follow the standard testing protocol introduced in ~\cite{Girard18,Zampieri18} by randomly and independently perturbing the accurate annotations on $3$ images of the city of San Francisco.
In contrast to generating displacement maps in~\cite{Girard18,Zampieri18}, we consider instance-level transformations.
The testing labels are generated by randomly and independently perturbing the accurate annotation instances 
to  achieve an error comparable to that of~\cite{Girard18,Zampieri18}. 
As shown in \Cref{Fig.align_result}, the \emph{AutoCorrect} approach outperforms all previous methods at all thresholds (in pixels).
This is because our method outputs transformation parameters for each instance independently, whereas prior works outputs a displacement field map, which is less expressive.
Furthermore, our consistency loss also works as a form of data augmentation which counters the small size of the \emph{INRIA buildings} dataset, further improving the performance.

Note that we learn to correct random and \emph{different} perturbations of objects that co-occur in the same image,
therefore, our proposed local~(per-object) correction is a better match to the type of errors observed in practice in aerial datasets as the location of the shifted annotations can be random and uncorrelated.
\subsection{Qualitative results}
As label noise of Satellite imagery is random, each instance label must be considered and corrected individually.
Our \emph{AutoCorrect} models deal with geometric alignments on an 
arbitrary number of instances, by aligning each instance sequentially.  
\Cref{fig.progression} shows the \emph{AutoCorrect} correction progression on testing data. 
Since the noise of each label is random, our model handles  it by iteratively correcting each semantic labels individually.
\Cref{fig.build_align1} shows  the  \emph{AutoCorrect} final predictions on a number of examples
from the test data. 

\input{fig_tab/Fig-progression}
\input{fig_tab/Fig-Alignment}

%% file: fig_tab/table1.tex
\begin{table}[t]
\centering\footnotesize
\begin{floatrow}
\capbtabbox{%
\setlength{\tabcolsep}{2pt}
\begin{tabular}{crrlccr}
\toprule
Model & Data & Noise      &         & SMM          & Consist.     & IoU \\
\midrule
A     & 3k   & 0\%        & ---     & $\times$     & $\times$     & 0.321\\
B     & 3k   & 0\%        & ---     & $\checkmark$ & $\times$     & 0.425\\
C     & 3k   & 0\%        & ---     & $\checkmark$ & $\checkmark$ & 0.436\\
D     & 3k   & 20\%       & Synth.  & $\checkmark$ & $\times$     & 0.404\\
E     & 3k   & 20\%       & Synth.  & $\checkmark$ & $\checkmark$ & 0.429\\
F     & 3k   & 40\%       & Synth.  & $\checkmark$ & $\times$     & 0.369\\
G     & 3k   & 40\%       & Synth.  & $\checkmark$ & $\checkmark$ & 0.381\\
H     & 20k  & $\sim$40\% & Natural & $\checkmark$ & $\times$     & 0.417\\
I     & 20k  & $\sim$40\% & Natural & $\checkmark$ & $\checkmark$ & 0.435\\
J     & 35k  & $\sim$40\% & Natural & $\checkmark$ & $\checkmark$ &  0.445\\
\bottomrule
\end{tabular}%
}{\vspace{-1em}%
\caption{Railway tracks dataset results. 
The SMM and Consist.\  refers to the spatial memory map and consistency loss respectively.}\label{Fig.results}%
}
\ffigbox{%
\includegraphics[width=.5\textwidth]{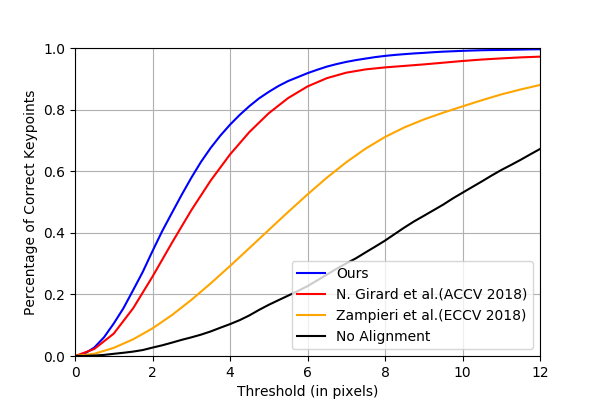}%
}{%
\caption{INRIA buildings dataset results. We outperform all recent works; from around 10 pixels threshold, our result is 100\% (i.e.~it cannot be improved further).}\label{Fig.align_result}%
}
\end{floatrow}
\end{table}

%% file: fig_tab/Fig-progression.tex
\begin{figure*}[!htb]
\centering\includegraphics[width=0.93\textwidth]{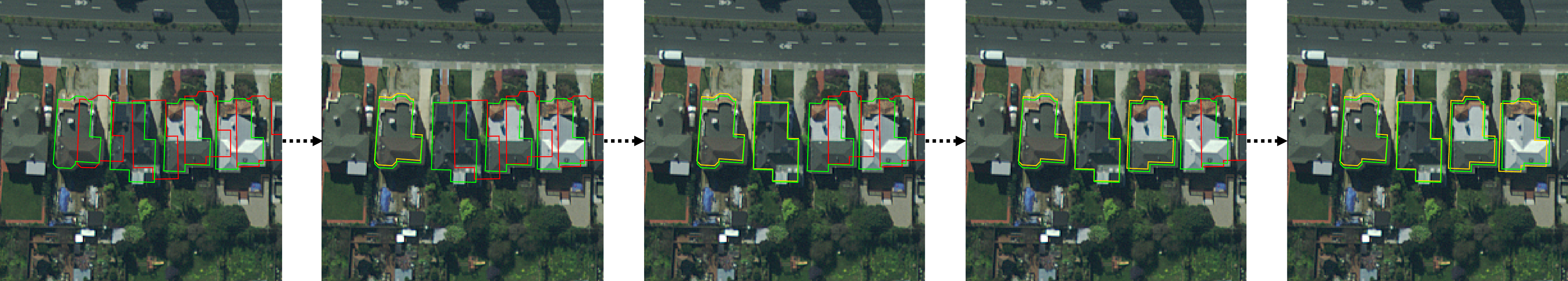}
\vspace{-1.2em}
\caption{Correction progression. 
Red polygons refer to noisy labels, green to noise-free labels, and yellow are our predictions. 
Noisy annotations are cleaned inductively.} \label{fig.progression}
\end{figure*}

%% file: fig_tab/Fig-Alignment.tex
\begin{figure*}[!htb]
\centering
\subfigure[]{\includegraphics[width=0.23\textwidth]{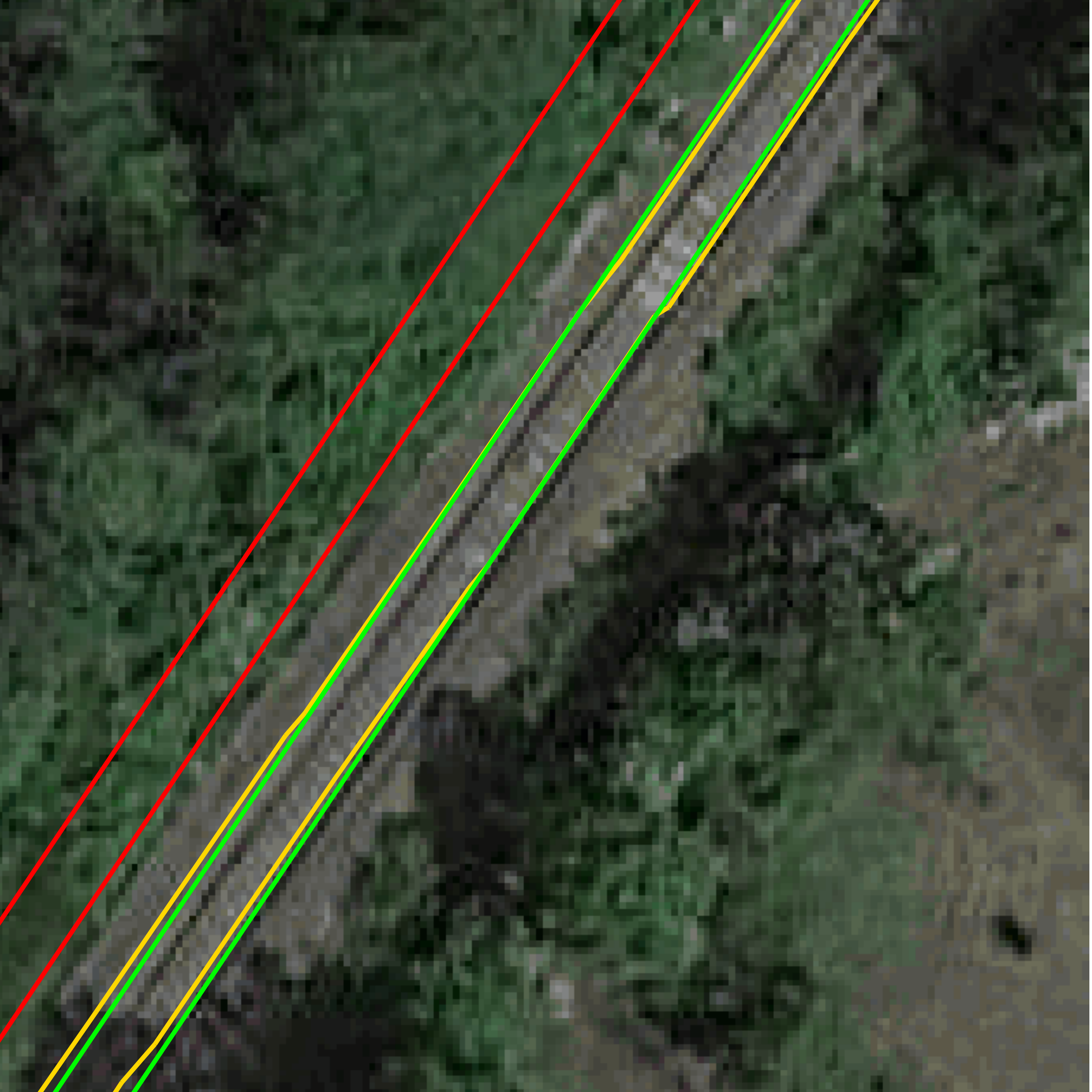}\label{Fig.rail4}}
\subfigure[]{\includegraphics[width=0.23\textwidth]{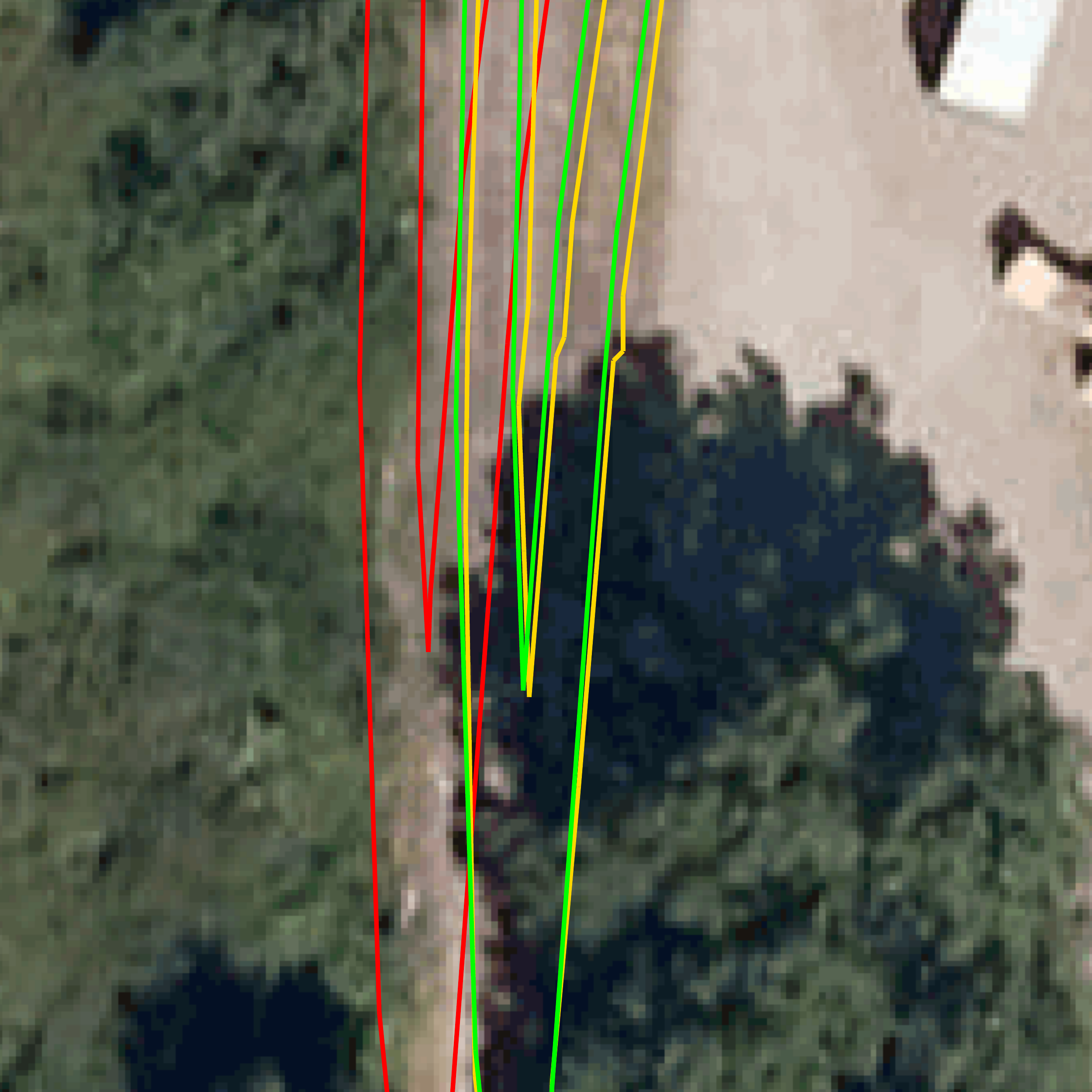}\label{Fig.rail3}}
\subfigure[]{\includegraphics[width=0.23\textwidth]{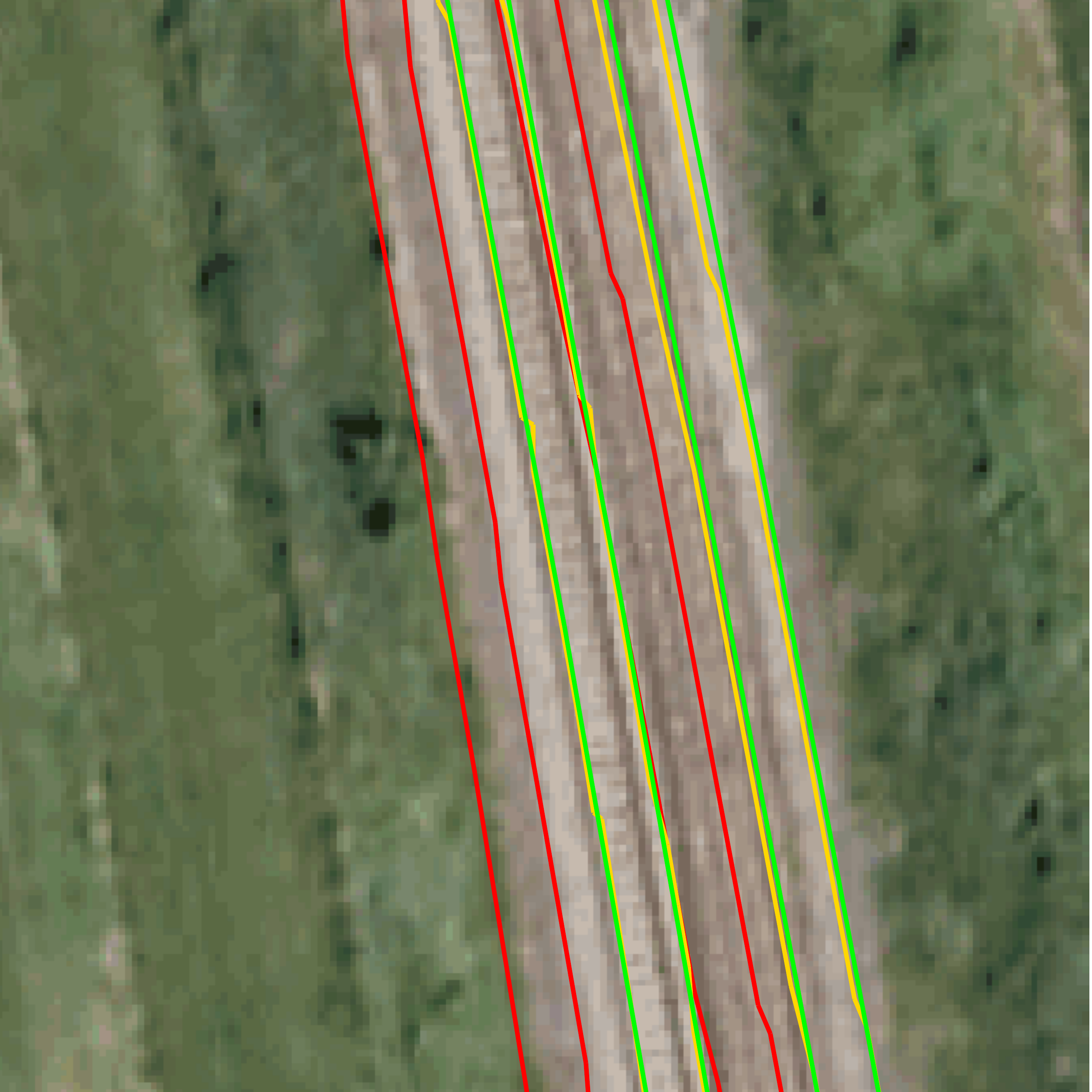}\label{Fig.rail1}}
\subfigure[]{\includegraphics[width=0.23\textwidth]{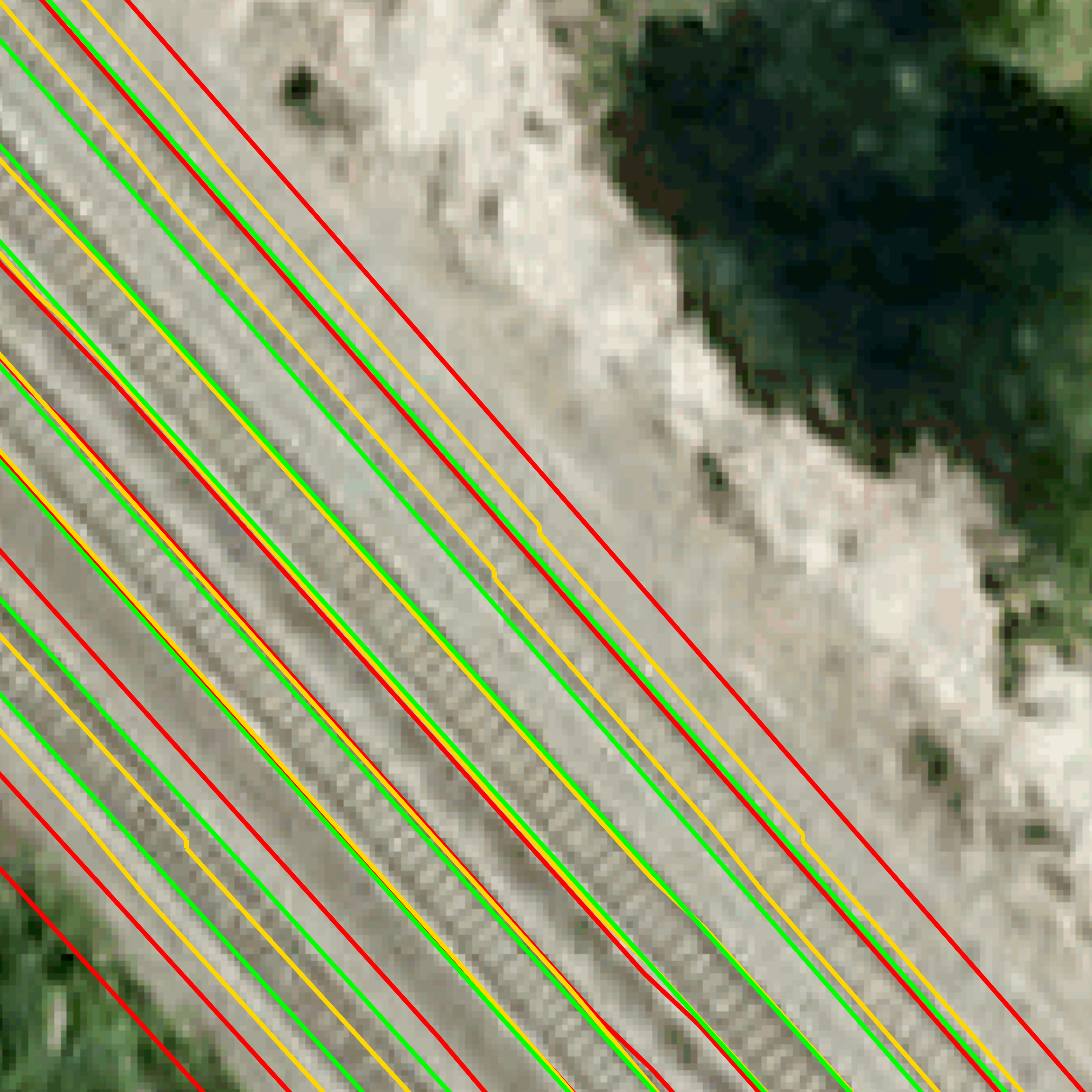}\label{Fig.rail2}}
\subfigure[]{\includegraphics[width=0.23\textwidth]{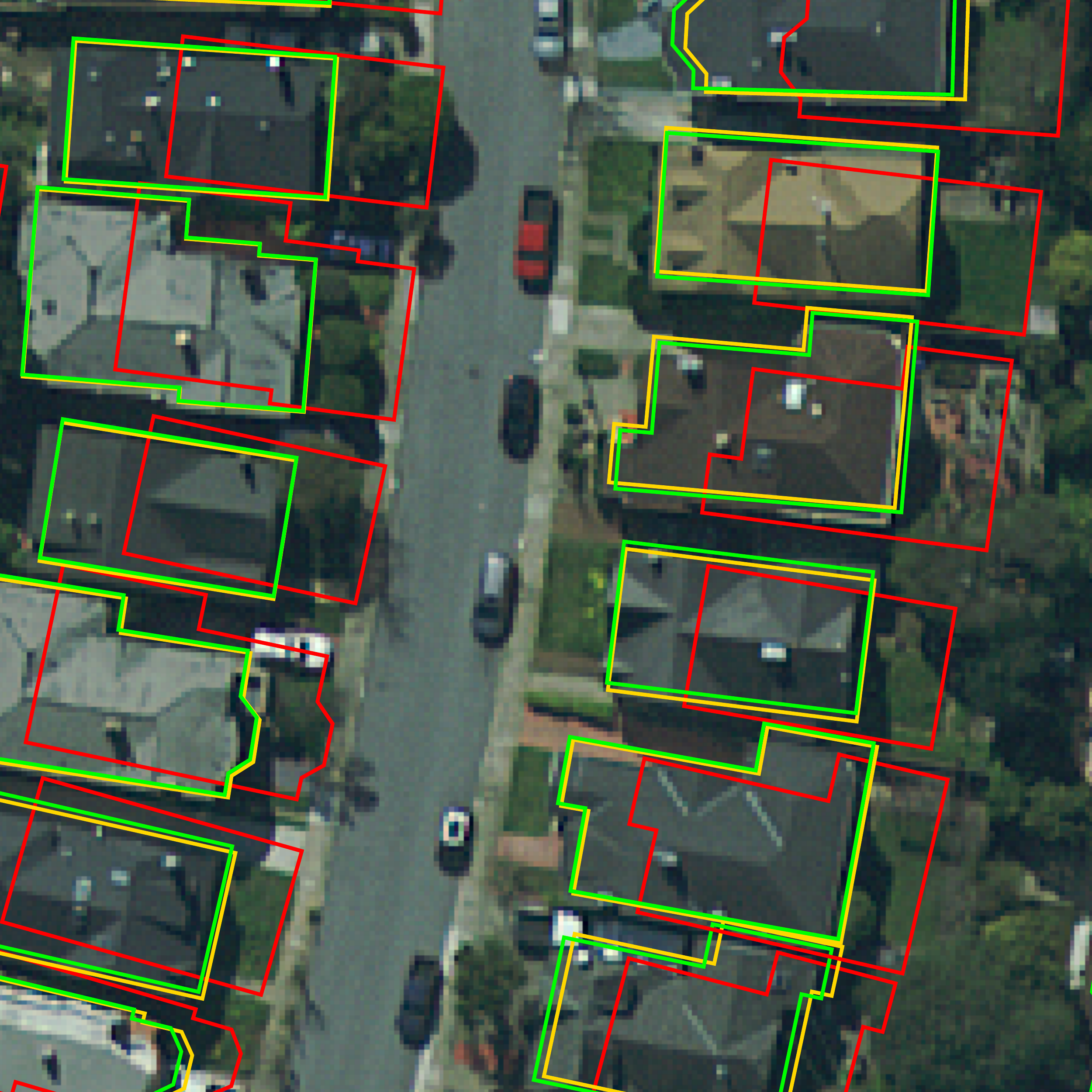}\label{Fig.build1}}
\subfigure[]{\includegraphics[width=0.23\textwidth]{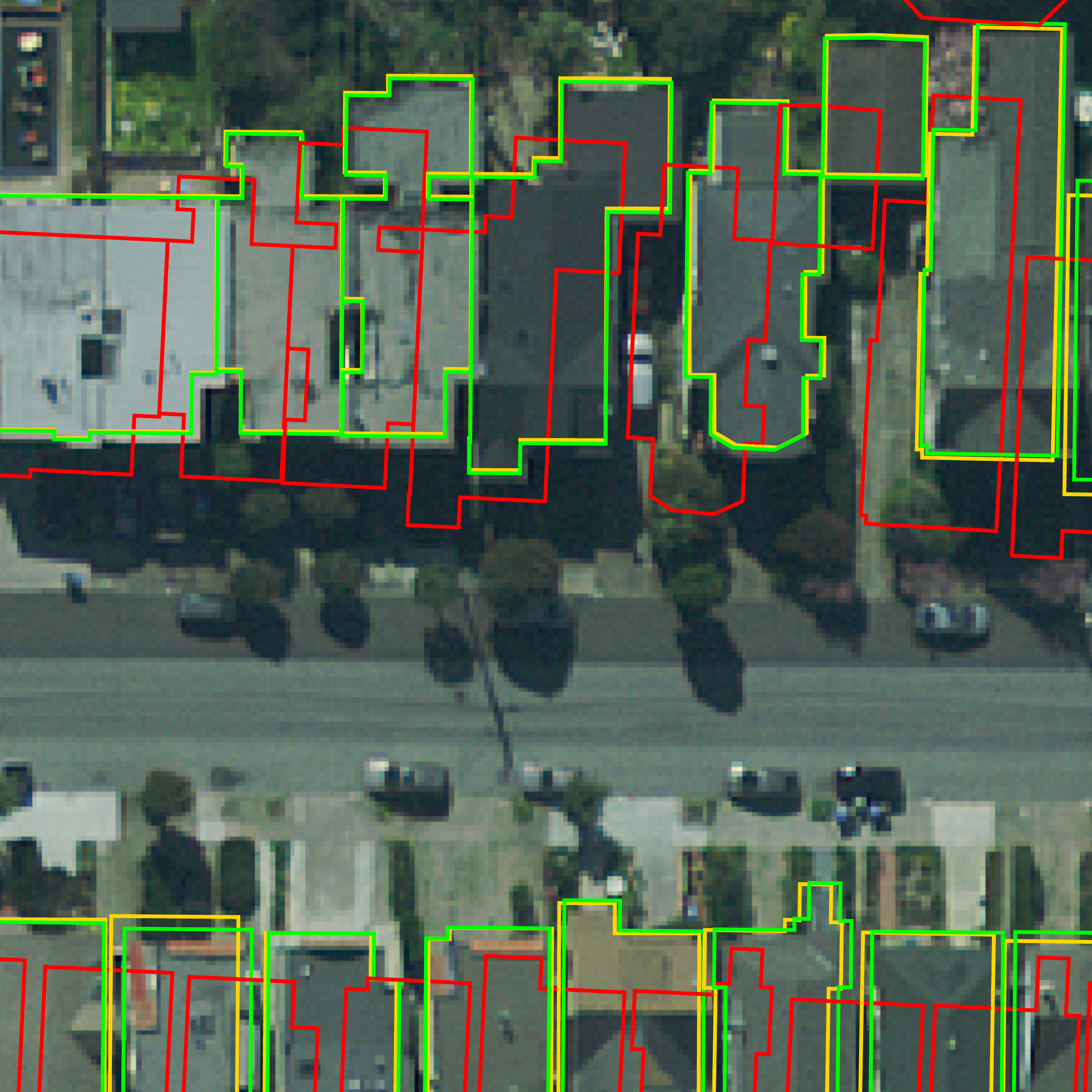}\label{Fig.build2}}
\subfigure[]{\includegraphics[width=0.23\textwidth]{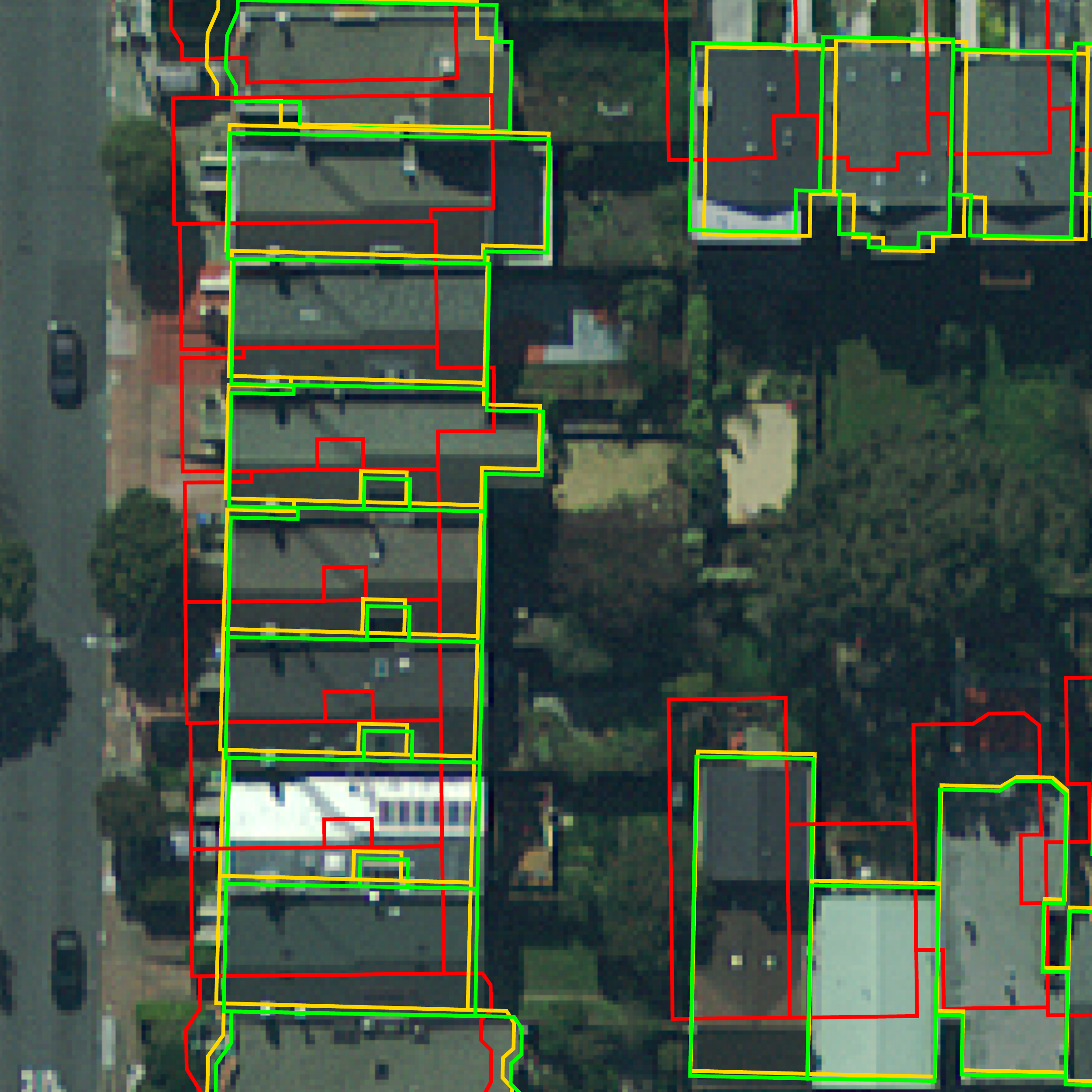}\label{Fig.build4}}
\subfigure[]{\includegraphics[width=0.23\textwidth]{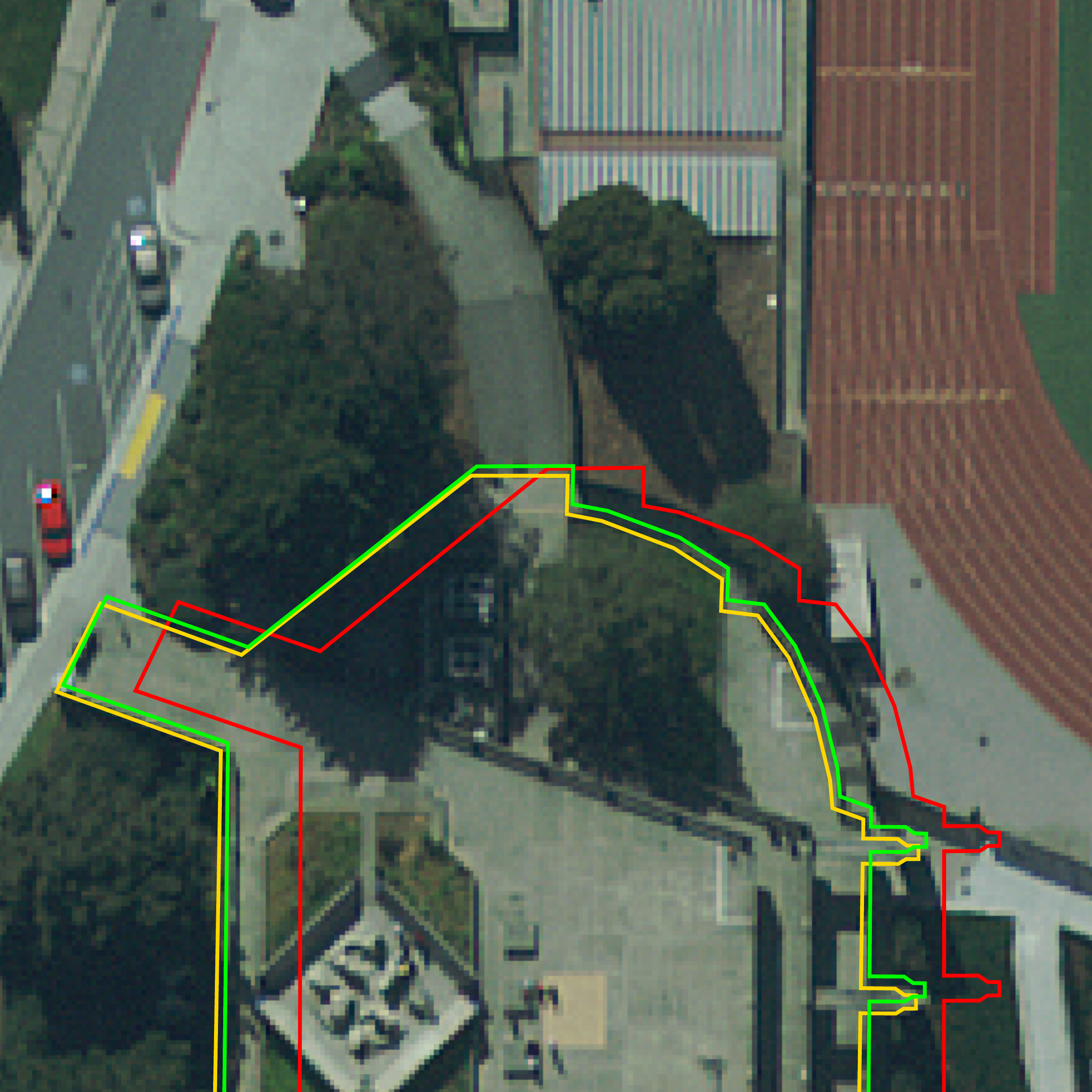}\label{Fig.build3}}
\vspace{-1.5em}
\caption{
Alignment results  for the \emph{Railway
tracks} dataset examples (top row), and \emph{INRIA
buildings} dataset (bottom row). 
The label noise of each instance (denoted in red) is random \ie local transformation of each instance is needed. Our predictions (denoted in yellow) achieve accurate correction comparing to ground truth (Green), by predicting a transformation on each instance.
In reality, \emph{AutoCorrect} can correct both noisy instance with regular shape (a), 
as well as instances with complex shapes ((b) \& (h)).
Figures (c) and (d) 
illustrate the capability of correcting an arbitrary number of
instances.}.
\label{fig.build_align1}     
\end{figure*}

%% file: sec/Conclusion.tex
\section{Conclusion}\label{Conclusion}

The \emph{AutoCorrect} method is based on three ideas:
a  spatial memory map that enables annotations to be adjusted sequentially
while taking into account the other annotations and their corrections,
a consistency loss that enables the model to be trained without the
knowledge of any noise-free annotation, and a self-supervised loss
that generates training data automatically.  \emph{AutoCorrect} outperforms
previously-published works and can
learn to correct almost for free from a large dataset where 40\% of the annotations
are heavily distorted, and obtain results  that are comparable to
approaches that require noise-free annotations.  Finally, we have
introduced the new \emph{Railway tracks} benchmark.